\documentclass{article}
\usepackage{arxiv}

\usepackage{fullpage}
\usepackage{cite}
\usepackage{graphicx}

\usepackage{amsmath}
\usepackage{multirow, rotating}
\usepackage{xspace}
\usepackage{mathtools}

\usepackage{amsfonts,amssymb}

\usepackage{amsmath}
\usepackage{amsthm}

\usepackage{geometry}

\usepackage[vlined,linesnumbered,ruled,titlenotnumbered]
{algorithm2e}
\usepackage{xcolor, soul}
\usepackage{float}
\usepackage{colortbl,dsfont}
\sethlcolor{yellow}
\newcommand{\cmmnt}[1]{}
\sethlcolor{lightgray}
\usepackage{xcolor}   
\usepackage{color} 
\newcommand{\gsemo}{GSEMO\xspace}
\newcommand{\gsemotD}{GSEMO3D\xspace}
\newcommand{\swgsemotD}{SW-GSEMO3D\xspace}
\newcommand{\fgsemotD}{Fast SW-GSEMO3D\xspace}

\newcommand{\pmax}{P_{\max}}

\newcommand{\ie}{i.\,e.\xspace}

\newcommand{\R}{\mathds{R}}

\newtheorem{lemma}{Lemma}
\newtheorem{theorem}{Theorem}

\usepackage{amsfonts}
\newcommand{\eg}{e.\,g.\xspace}

\newcommand{\temax}{t_{\mathrm{max}}}
\newcommand{\cemax}{c_{\mathrm{max}}}
\newcommand{\tefrac}{t_{\mathrm{frac}}}
\newcommand{\mumin}{\mu_{\min}}
\newcommand{\cmax}{c_{\max}}

\newcommand{\std}{\mathrm{std}}

\newcommand{\N}{\mathbb{N}}

\AtBeginDocument{%
  \providecommand\BibTeX{{%
    \normalfont B\kern-0.5em{\scshape i\kern-0.25em b}\kern-0.8em\TeX}}}

\title{Sliding Window 3-Objective Pareto Optimization for Problems with Chance Constraints}

    \author{Frank Neumann\\
Optimisation and Logistics\\
School of Computer and Mathematical Sciences\\
The University of Adelaide\\
Adelaide, Australia
\And
Carsten Witt\\
Algorithms, Logic and Graphs\\
DTU Compute\\ Technical University of Denmark\\
2800 Kgs. Lyngby Denmark
}

\begin{document}

\maketitle              % typeset the header of the contribution
\begin{abstract}
Constrained single-objective problems have been frequently tackled by evolutionary multi-objective algorithms where the constraint is relaxed into an additional objective. Recently, it has been shown that Pareto optimization approaches using bi-objective models can be significantly sped up using sliding windows~\cite{NeumannWittECAI23}. In this paper, we extend the sliding window approach to $3$-objective formulations for tackling chance constrained problems. On the theoretical side, we show that our new sliding window approach improves previous runtime bounds obtained in \cite{NeumannWittGECCO23} while maintaining the same approximation guarantees. Our experimental investigations for the chance constrained dominating set problem show that our new sliding window approach allows one to solve much larger instances in a much more efficient way than the 3-objective approach presented in \cite{NeumannWittGECCO23}.

\end{abstract}
\section{Introduction}
%chance-constrained optimization
%submodular problems
%background
%why we are doing this
Multi-objective formulations have been widely used to solve single-objective optimization problems. The initial study carried out by Knowles et al.~\cite{DBLP:conf/emo/KnowlesWC01} for the H-IFF and the traveling salesperson problem shows that such formulations can significantly reduce the number of local optima in the search space and uses the term \emph{multi-objectivization} for such approaches.
Using multi-objective formulations to solve constrained single-objective optimization problems by evolutionary multi-objective optimization using the constraint as an additional objective has shown to be highly beneficial for a wide range of problems~\cite{DBLP:journals/nc/NeumannW06,DBLP:journals/ec/FriedrichHHNW10,DBLP:journals/algorithmica/KratschN13}.
Using the constraint as an additional objective for such problems allows simple evolutionary multi-objective algorithms such as GSEMO mimic a greedy behaviour and as a consequence allows us to achieve theoretically best possible performance guarantees for a wide range of constrained submodular optimization problems~\cite{DBLP:conf/nips/QianYZ15,DBLP:conf/ijcai/QianSYT17,DBLP:journals/ai/RoostapourNNF22}. Such approaches have been widely studied recently under the term \emph{Pareto optimization} in the artificial intelligence and machine learning literature~\cite{DBLP:books/sp/ZhouYQ19}.

In the context of problems with stochastic constraints, it has recently been shown that $3$-objective formulations where the given constraint is relaxed into a third objective lead to better performance than 2-objective formulations that optimize the expected value and variance of the given stochastic components under the given constraint~\cite{NeumannWittGECCO23,DBLP:conf/ijcai/0001W22}. 
The experimental investigations for the chance constrained dominating set problem carried out in \cite{NeumannWittGECCO23} show that the 3-objective approach is beneficial and outperforms the bi-objective one introduced in \cite{DBLP:conf/ijcai/0001W22} for medium size instances of the problem. However, it has difficulties in computing even a feasible solution for larger graphs.
In order to deal with large problem instances we design a new $3$-objective Pareto optimization approach based on the sliding window technique for Pareto optimization recently introduced in~\cite{NeumannWittECAI23}. Using sliding window selection has been shown to scale up the applicability of \gsemo type algorithms for the optimization of monotone functions under general cost constraints. 
Here, at a given time step only solutions with a fixed constraint value are chosen in the parent selection step. This allows the algorithm to proceed with achieving progress in the same way as the analysis for obtaining theoretical performance guarantees. 
It does so by dividing the given function evaluation budget $\temax$ equally among the different constraint values starting with selecting individuals with the smallest constraint value at the beginning and increasing it over time until it reaches that given constraint bound at the end of the run.
A positive effect is that the maximum population size can be eliminated as a crucial factor in the given runtime bounds. Furthermore, experimental studies carried out in \cite{NeumannWittECAI23}  show that the approach provides major benefits when solving problems on graphs with up to 21,000 vertices.

Making the sliding window technique work requires one to deal with a potentially large number of trade-off solutions even for a small number of constraint values. We design highly effective $3$-objective Pareto optimization approaches based on the sliding window technique. Our theoretical investigations using runtime analysis for the chance constrained problem under a uniform constraint show that our approach may lead to a significant speed-up in obtaining all required Pareto optimal solutions. In order to make the approach even more suitable in practice, we introduce additional techniques that do not hinder the theoretical performance guarantees, but provide additional benefits in applications. One important technique is to control the sliding window through an additional parameter~$a$. Choosing $a \in \mathopen{]}0,1\mathclose{[}$ allows the algorithm to move the sliding window faster at the beginning of the optimization process and slow down when approaching the constraint bound. This allows us to maintain the benefit of the Pareto optimization approach including its theoretical performance guarantees while focusing more on the improvement of already high quality solutions at the end of the optimization run. The second technique that we incorporate is especially important for problems like the dominating set problem where a constraint that is not fulfilled for most of the optimization process needs to be fulfilled at the end. In order to deal with this, we introduce a parameter $\tefrac \in [0,1]$ which determines the fraction of time our sliding window technique is used. If after $\tefrac \cdot \temax$ steps a feasible solution has not been found yet, then in each step a solution from the population that is closest to feasible is selected in the parent selection step to achieve feasibility within the last $(1-\tefrac) \cdot \temax$ steps.

This paper is structured as follows: in Section~\ref{sec:algorithms}, we present the
multi-objective algorithms considered in this paper, 
in particular the 3-objective approach using sliding 
window selection. Section~\ref{sec:runtime} proves 
the improved runtime bounds for this 
approach. Section~\ref{sec:experiments} presents 
the empirical comparison of the different algorithms 
on a large set of instances of the minimum 
dominating set problem. We finish with some conclusions.

\section{Algorithms}
\label{sec:algorithms}
In this section, we define the algorithmic framework incorporating sliding window selection into two-objective optimization problems under constraints. It combines the  $3$-objective problem formulation from \cite{NeumannWittGECCO23}, where the underlying problem is $2$-objective and a constraint is converted to a helper objective, with the two-objective formulation  from \cite{NeumannWittECAI23},  where the problem is single-objective and the constraint is converted to a helper objective and additionally undergoes the so-called sliding window selection. More precisely, sliding window is based on the observation that several problems under uniform 
constraints can be solved 
by iterating over increasing constraint values and optimizing the actual objectives for each fixed constraint value.

We consider an optimization problem on bit strings $x\in\{0,1\}^n$ involving two objective
functions 
$\mu(x),v(x)\colon\{0,1\}^n \to \R^{+}_0$ and an integer-valued
constraint function $c(x)\colon \{0,1\}^n \rightarrow \N$ with bound~$B$, \ie, the only solutions satisfying $c(x)\le B$ are feasible. 
Our new approach called \swgsemotD is shown 
in Algorithm~\ref{alg:GSEMO-sliding} (which will later be extended to 
Algorithm~\ref{alg:FGSEMO-sliding} explained below). 
The sliding window selection in Algorithm~\ref{alg:select-sliding} 
will be used as a module in \swgsemotD  and choose from its current population~$P$, which is the first parameter of the algorithm. The idea is to select only from a subpopulation of constraint values in a specific interval determined 
by the maximum constrained value~$B$, the current generation~$t$, the maximum number of iterations of the algorithm $\temax$, and further parameters. In the simplest case (where the remaining parameters 
are set to~$a=1$,  $\cmax=-1$ and  $\tefrac=1$), the time interval $[1,\temax]$ is uniformly divided into $B$ time intervals in which only the subpopulation having constraint values in the interval $[\lfloor(t/\temax)B\rfloor-\std,  \lfloor(t/\temax)B\rfloor+\std]$, 
where $\std\ge 0$ is a deviation that allows selection from a larger interval, which is another 
heuristic component investigated in Section~\ref{sec:experiments}. Moreover, as not all problems may benefit from selecting according to the specific
interval order, the calls to Algorithm~\ref{alg:select-sliding} resort to selection from the interval $[B-\std,B]$ for the last $(1-\tefrac)\temax$ steps. Finally, since making progress may become increasingly difficult for increasing constraint values, the selection 
provides the parameter~$a$ which will allow time intervals of varying length for the different constraint values to choose from. If $a<1$, the time allocated to choosing from a specific constraint value (interval) increases with the constraint value. 
Lines~8--10 of the algorithm make sure that solutions 
with too low constraint value (less than~$\ell$), but not equaling the parameter 
$\cmax$ are permanently 
removed from the population. Line~$11$ confines the population to 
select from to the desired window of constraint values $[\ell,h]$. 
In case that no solution of those values exists, a uniform choice from the population 
remaining after removal of individuals of 
too low constraint values is made. 
Hence, even if there are no individuals
with constraint values in the interval
$[\ell,h]$, then lines~8--10 favor increasing constraint values.

\begin{algorithm}[t]
Choose initial solution $x \in \{0,1\}^n$\;
Set $t_0 \leftarrow -1$\;
 $P\leftarrow \{x\}$\;
 Compute $f(x) = (\mu(x), v(x), c(x))$\;
 $t \leftarrow 1$\;
 $\mu_{\min} \leftarrow \mu(x)$\;
 \If{$\mu_{\min}=0$}{
 $t_0 \leftarrow t$\;
 }
\Repeat{$\mathit{t\geq \temax}$}{
\If{$(t_0 = -1) \wedge ( t \leq  \temax)$}{
$x  \leftarrow \arg \min\{\mu(z) \mid z \in P\}$ (breaking ties arbitrarily)
}
\Else{ $x=\text{sliding-selection}(P, t-t_0, \temax-t_0, 0, B,1, 1, -1)$\;
}
Create $y$ from $x$ by mutation\;
 Compute $f(y) = (\mu(y), v(y), c(y))$\;
 \If{$\mu(y)<\mu_{\min}$}
 {$\mu_{\min} \leftarrow \mu(y)$\;
 
 }
\If{$(t_0=-1) \wedge (\mu_{\min}=0)$}{
$t_0 \leftarrow t$\;
}
 
\If{$\nexists\, w \in P: w \succ y$} {
  $P \leftarrow (P \setminus \{z\in P \mid y \succeq z\}) \cup \{y\}$\;
  }
  $t\leftarrow t+1$\;
  }
\caption{Sliding Window GSEMO3D (SW-GSEMO3D)} \label{alg:GSEMO-sliding}
\end{algorithm}
In our theoretical studies, we focus on  \swgsemotD which uses sliding window selection  with the default choices~$\std=0$, $\tefrac=1$, $a=1$ and $\cmax=-1$. 
It starts out with a solution chosen 
uniformly at random 
and is run on the bi-objective optimization problem $(\mu(x),v(x)$), both of which are minimized and will correspond to expected value and variance of a chanced constrained optimization problem further explained in Section~\ref{sec:runtime}. In particular, we assume $\mu(x)\ge 0$ for 
all $x\in\{0,1\}^n$ and $\mu(0^n)=0$, and accordingly for $v(x)$.
Following the usual definitions in multi-objective optimization, 
we say that a solution $x$ dominates a 
solution $y$ ($x \succeq y$) iff $c(x) \geq c(y) \wedge \mu(x) \leq \mu(y) \wedge v(x) \leq v(y)$. Furthermore, we say a solution 
$x$ strongly dominates $y$ ($x \succ y$) iff $x \succeq y$ and $(\mu(x),v(x),c(x)) \not = (\mu(x),v(x),c(x))$.

\begin{algorithm}[t]
$y \leftarrow x$\;
\Repeat{$x \not =y$}{
Create $y$ from $x$ by flipping each bit $x_{i}$ of $x$ with probability $\frac{1}{n}$.
}
Return $y$\;
\caption{Standard-bit-mutation-plus(x)} \label{alg:mutation-plus}
\end{algorithm}

\begin{algorithm}[t]
$\hat{c} \leftarrow (t^a/(\tefrac \cdot \temax)^a)\cdot B$\;
\If{$t \leq (\tefrac \cdot \temax)$}{
$\ell = \lfloor\hat{c} \rfloor-std$\;
$h = \lceil \hat{c} \rceil +std$\;
}
\Else{
$\ell = B- std$\;
$h = B$\;
}
%$\cemax  \leftarrow \max\{c(x) \mid x \in P\}$\; 
\For{$x \in P$}{
\If{$(c(x) < \ell) \wedge (c(x)\not =\cemax) \wedge (\cemax\not= -1) \wedge (|P|>1)$}{
$P \leftarrow P \setminus \{x\}$; 
}
}
$\hat{P} =\{x \in P \mid \ell \leq c(x) \leq h \}$\;
\If{$\hat{P}=\emptyset$}{$\hat{P} \leftarrow P$\;}
Choose $x\in \hat{P}$ uniformly at random\;
Return $x$\;
\caption{sliding-selection$(P, t, \temax, std, B, \tefrac, a, \cemax)$} \label{alg:select-sliding}
\end{algorithm}

The \swgsemotD starts out with a solution~$x\in\{0,1\}^n$ chosen by the user, \eg, as the all-zeros string or uniformly at random. It works in two phases. As long as the minimum $\mu$-value of the population called $\mumin$ is positive, 
it chooses a solution of this smallest $\mu$-value, applies mutation, usually standard bit mutation avoiding duplicates (Algorithm~\ref{alg:mutation-plus}), and accepts the offspring into the population if it is not strictly dominated by another 
member of the population.
All 
individuals that are weakly dominated by the offspring are then removed 
from the population. In any case, the 
current population always consists of mutually non-dominating solutions only. 
From the point of time~$t_0$ on where a solution~$x$ satisfying $\mu(x)=0$ is found for the first time, the 
algorithm chooses from the population using sliding window selection (see Algorithm~\ref{alg:select-sliding}) for the remaining  $\temax-t_0$ steps. In Algorithm~\ref{alg:select-sliding}, the choice $\cemax=-1$ implies 
that lines 8--10 do nothing.

Algorithm~\ref{alg:FGSEMO-sliding} called \fgsemotD extends Algorithm~\ref{alg:GSEMO-sliding} with heuristic elements as follows. First of all, sliding window selection is called with user-specified choices of 
$\std$, $\tefrac$ and~$a$ as defined above. 
Moreover, it keeps track of the maximum constraint value~$c_{\max}$ found in the population (lines 24--25), uses that in the sliding window selection
and introduces a margin parameter~$\epsilon$ such that 
sliding window selection is only run until $c_{\max}=B-\epsilon$. Afterwards, \ie, when the algorithm is close to the constraint boundary, making further progress in the constraint value may be too difficult for sliding window selection. Therefore, 
for the last $(1-\tefrac)\temax$ steps,  the algorithm chooses  an individual of maximum constraint value if $c_{\max}<B-\epsilon$ holds. 
 These heuristic elements underlying the parameters $\std$, $\tefrac$, $a$ and $\epsilon$ and the use of $\cmax$ in 
the sliding window selection will show 
some empirical benefit in Section~\ref{sec:experiments}.

For the sake of completeness, we also define \gsemo, a classical multi-objective optimization algorithm \cite{GielCEC2003,DBLP:journals/tec/LaumannsTZ04} that has inspired the developments of 
Algorithms~\ref{alg:GSEMO-sliding} and~\ref{alg:FGSEMO-sliding} and serves as a baseline in our experiments. It maintains a population of
non-dominating solutions of unbounded size, starting from a solution chosen 
uniformly at random, and creates one offspring per generation by choosing an individual uniformly at random, applying standard bit mutation avoiding duplicates, and accepting the offspring if it is not 
dominated by any member of the population. Depending on the 
number of objectives used in the
experiments in Section~\ref{sec:experiments}, we will consider specific 
instances of the 
algorithm called GSEMO2D and GSEMO3D as 
in \cite{NeumannWittGECCO23}.

\begin{algorithm}[t]
Choose initial solution $x \in \{0,1\}^n$\;
$t_0 \leftarrow -1$, $t \leftarrow 1$, $\mu_{\min} \leftarrow \mu(x)$, $\cemax \leftarrow -1$\;
 $P\leftarrow \{x\}$\;
 Compute $f(x) = (\mu(x), v(x), c(x))$\;
 \If{$(c(x)> c_{\max}) \wedge (c(x) \leq B)$}
 {$c_{\max} \leftarrow   c(x)$\;
 }
 \If{$\mu_{\min}=0$}{
 $t_0 \leftarrow t$\;
 }
\Repeat{$\mathit{t \geq \temax}$}{
$t \leftarrow t+1$\;

\If{$(t_0 = -1) \wedge ( t \leq \tefrac \cdot \temax)$}{
$x  \leftarrow \arg \min\{\mu(z) \mid z \in P\}$
}
\Else{ \If{$( t > \tefrac \cdot \temax) \wedge (c_{\max}<B - \epsilon)$ 
}{
$x  \leftarrow \arg \max\{c(z) \mid z \in P\}$ 
}

\Else
{
$x=\text{sliding-selection}(P, t-t_0, \temax-t_0, std, B, \tefrac,a, \cemax)$\;
}
}
Create $y$ from $x$ by mutation\;
 Compute $f(y) = (\mu(y), v(y), c(y))$\;
 \If{$\mu(y)<\mu_{\min}$}
 {$\mu_{\min} \leftarrow  \mu(y)$\;
 }
 \If{$(t_0=-1) \wedge (\mu_{\min}=0)$}{
$t_0 \leftarrow t$\;
}
 
 \If{$(c(y)> c_{\max}) \wedge (c(y) \leq B)$}
 {$c_{\max} \leftarrow   c(y)$\;
 }
 
\If{$\nexists\, w \in P: w \succ y$} {
  $P \leftarrow (P \setminus \{z\in P \mid y \succeq z\}) \cup \{y\}$\;
  }
  }
\caption{Fast Sliding-Window GSEMO3D (Fast SW-GSEMO3D) (Parameters: $\temax, \tefrac, std, a, \epsilon$)} \label{alg:FGSEMO-sliding}
\end{algorithm}

\begin{algorithm}[h]
Choose initial solution $x \in \{0,1\}^n$\;
 $P\leftarrow \{x\}$\;
\Repeat{$\mathit{stop}$}{
Choose $x\in P$ uniformly at random\;
Create $y$ from $x$ by mutation\;
\If{$\nexists\, w \in P: w \succ y$} {
  $P \leftarrow (P \setminus \{z\in P \mid y \succeq z\}) \cup \{y\}$\;
  } }
\caption{Global simple evolutionary multi-objective optimizer (GSEMO)} \label{alg:GSEMO}
\end{algorithm}

\section{Runtime Analysis of 3D Sliding Window Algorithm}
\label{sec:runtime}

In our theoretical study, we 
consider the chance constrained problem investigated in \cite{DBLP:conf/ijcai/0001W22} using rigorous runtime analysis, which is a major direction in the area of theory of evolutionary computation~\cite{NeumannW10,Jansen13,DoerrN20}. 
Given a set of $n$ items $I=\{e_1, \dots, e_n\}$ with stochastic weights $w_i$, $1 \leq i \leq n$, we want to solve 
\begin{equation}
\min W  \text{~~~~subject to~~~~}  (\mathit{Pr}( w(x) \leq W) \geq \alpha) \wedge (|x|_1 \geq k),
\label{chance-problem}
\end{equation}
where $w(x) = \sum_{i=1}^n w_i x_i$, $x \in \{0,1\}^n$, and $\alpha \in \mathopen{[}1/2,1\mathclose{[}$.
The weights $w_i$ are independent 
random variables  following a normal distribution $N(\mu_i, \sigma_i^2)$, $1 \leq i \leq n$, where $\mu_i \geq 1$ and $\sigma_i\geq 1$, $1 \leq i \leq n$.
We denote by $\mu(x) = \sum_{i=1}^n \mu_i x_i$ the expected weight and by $v(x) = \sum_{i=1}^n \sigma_i^2 x_i$ the variance of the weight of solution $x$.

As stated in \cite{DBLP:conf/ijcai/0001W22}, the problem given in Equation~\eqref{chance-problem} is equivalent to minimizing 
\begin{equation}
\label{eq:weight}
\hat{w}(x)=\mu(x) + K_{\alpha} \sqrt{v(x)},
\end{equation}
under the constraint that $|x|_1 \geq k$ holds. 
Here, $K_{\alpha}$ denotes the $\alpha$-fractional point of the standard Normal distribution.

Our algorithm can also be used to maximize a given deterministic objective $c(x)$ under a given chance constraint, i.\,e., 
\begin{equation}
\max c(x)  \text{~~~~subject to~~~~}  \mathit{Pr}( w(x) \leq B) \geq \alpha.
\label{chance-problem2}
\end{equation}
with $w(x) = \sum_{i=1}^n w_i x_i$ where each $w_i$ is chosen independently of the other according to a Normal distribution $N(\mu_i, \sigma_i^2)$, and $B$ and $\alpha \in [1/2, 1[$ are a given weight bound and reliability probability.
Such a problem formulation includes for example the maximum coverage problem in graphs with so-called chance constraints~\cite{DBLP:conf/aaai/DoerrD0NS20,DBLP:conf/ppsn/NeumannN20}, where $c(x)$ denotes the nodes  covered by a given solution $x$ and the costs are stochastic. Furthermore, the chance constrained knapsack problem as investigated in \cite{DBLP:conf/gecco/XieN020,DBLP:conf/gecco/XieHAN019} fits into this problem formulation.

In \cite{NeumannWittGECCO23}, 
the 3-objective formulation of chance-constrained optimization
problems under a uniform constraint 
given in~\eqref{chance-problem} was proposed. 
Let  
$
f_{3D}(x)= ( \mu(x), v(x), c(x)),
$
where $\mu(x)$ and $v(x)$ are the expected weight and variance, respectively, 
as above, and $c(x)$ is the constraint value of a given solution that should be maximized. In our theoretical study, we 
focus on the case $c(x) = |x|_1$, 
which 
turns the constraint $|x|_1 \geq k$ into the additional objective of maximizing the number of bits in the given bitstring. This 3-objective formulation 
was introduced as an alternative model to  the bi-objective model 
from \cite{DBLP:conf/ijcai/0001W22}, which considers penalty terms 
for violating the constraint $|x|_1 \geq k$. 

Based on the ideas for the 3-objective GSEMO from \cite{NeumannWittGECCO23}, we formulate the following result for \swgsemotD (Algorithm~\ref{alg:GSEMO-sliding}). The analysis 
is additionally 
inspired by\cite{NeumannWittECAI23}, where  a bi-objective sliding windows approach for submodular optimization was analyzed. Our theorem assumes 
an initialization with the all-zeros string. If uniform initialization is used, \swgsemotD nevertheless reaches the all-zeros string  efficiently, as 
shown in a subsequent lemma (Lemma~\ref{lem:swgsmtotduniform}).

The following theorem is based on the maximum population size $\pmax^{(i)}$ observed in any of the sliding window intervals. Note that when running the algorithm, the runtime for a given sliding window  can be adapted to $t^{(i)}_{\max}= \pmax^{(i)} n \ln n$ during the run based on the observed value of $\pmax^{(i)}$ in order to guarantee the stated approximation result. Note that 
the previous result from 
\cite{NeumannWittGECCO23} showed 
an upper bound of $O(n^2\pmax)$, 
where $\pmax$ is the overall 
maximum population size observed 
in the run of the algorithm. 
If the largest $\pmax^{(i)}$ 
is significantly smaller than $\pmax$, 
the following theorem gives a significantly
stronger upper bound.

\begin{theorem}
\label{theo:swgsemotdzeros}
Let $\pmax^{(i)}$ denote the largest number of individuals with constraint value~$i$ present in the population at all points in time where \swgsemotD can select such individuals, let $t^{(i)}_{\max} = \pmax^{(i)} n \ln n$ and let $\temax = 4en\max_{i=0}^{n-1}  t^{(i)}_{\max} $. 
    Then \swgsemotD, initialized with~$0^n$, computes a population which includes an optimal solution for the problem given in Equation~\eqref{chance-problem} (for any choice of~$k\in\{0,\dots,n\}$ and~$\alpha\in [1/2,1\mathclose{[}$) and Equation~\eqref{chance-problem2} (with $c(x)=|x|_1$ for any choice of $B\in\{0,\dots,n\}$ and $\alpha\in [1/2,1\mathclose{[}$) 
    until time $\temax = O(\max_{i=0}^{n-1}\{\pmax^{(i)}\} \cdot n^2\log n)$  with probability $1 - o(1)$.
\end{theorem}

\begin{proof}
Let $X^k = \{x \in \{0,1\}^n \mid |x|_1=k\}$ be the set of all solutions having exactly $k$ elements. 
We show the following more technical statement~$S$: \emph{the population~$P$ at time 
$\temax$ will, with the probability bound claimed in the theorem, contain for each 
$\alpha\in[1/2,1\mathclose{[}$ and $k\in\{0,\dots,n\}$ a solution 
\begin{equation}
x_{\alpha}^k = \arg \min_{x \in X^k} \left\{ \mu(x) + K_{\alpha} \sqrt{v(x)} \right\},\label{eq:optimality-included} 
\end{equation}
\ie, 
$P \supseteq \{x^k_{\alpha} \mid 0 \leq k \leq n, \alpha \in [1/2, 1\mathclose{[}\}$.} 
By Theorem~4.3 in \cite{NeumannWittGECCO23}, such a population 
contains the optimal solutions for any
choice of $\alpha\in[1/2,1\mathclose{[}$.
Note that not the whole set of Pareto optimal solutions is necessarily required.

To show statement~$S$, we re-use the following definitions from 
 \cite{DBLP:conf/ijcai/0001W22}. Let $\lambda_{i,j} = \frac{\sigma_j^2 - \sigma_i^2}{(\mu_i-\mu_j) +(\sigma^2_j - \sigma^2_i)}$ for the pair of elements $e_i$ and $e_j$ of the given input where $\sigma^2_i < \sigma^2_j$ and $\mu_i > \mu_j$ holds, $1 \leq i < j \leq n$. 
The set $\Lambda=\{\lambda_0, \lambda_1, \ldots, \lambda_{\ell}, \lambda_{\ell+1}\}$ where $\lambda_1, \ldots, \lambda_{\ell}$ are the values $\lambda_{i,j}$ in increasing order and 
$\lambda_0=0$ and $\lambda_{\ell+1}=1$.  Moreover, 
 we define the function
$
f_{\lambda}(x) = \lambda \mu(x) + (1-\lambda) v(x) $
and also use it applied to elements $e_i$, i.\,e. 
$
 f_{\lambda}(e_i) = \lambda \mu_i + (1-\lambda) \sigma_i^2 .
$

As noted in \cite{DBLP:conf/ijcai/0001W22}, 
for a given $\lambda$ and a given number~$k$ of elements to include, the function $f_{\lambda}$ can be optimized by a greedy approach which iteratively selects a set of $k$ smallest elements 
according to $f_{\lambda}(e_i)$.
For any $\lambda \in [0,1]$, an optimal solution for $f_{\lambda}$ with $k$ elements is Pareto optimal as there 
is no other solution with at least $k$ elements that improves the expected cost or variance without impairing the other. 
Hence, once obtained such a solution $x$, the resulting objective vector $f_{3D}(x)$ will remain in the population for the rest of the run of 
\swgsemotD. 
Furthermore, the set of optimal solutions for different $\lambda$ values only change at the $\lambda$ values of the set $\Lambda$ as these $\lambda$ values constitute the weighting where the order of items according to $f_{\lambda}$ can switch~\cite{DBLP:journals/dam/IshiiSNN81,DBLP:conf/ijcai/0001W22}. 

We consider a $\lambda_i \in \Lambda$ with $0 \leq i \leq \ell$. 
Similarly to \cite{DBLP:journals/dam/IshiiSNN81}, we
define $\lambda_i^* = (\lambda_i + \lambda_{i+1})/2$. The order 
of items according to the weighting of expected value and variance 
can only change at values $\lambda_i \in \Lambda$ and the resulting objective
vectors are not necessarily unique for values $\lambda_i \in \Lambda$.  
Choosing the $\lambda_i^*$-values in the defined way gives 
optimal solutions for all $\lambda \in [\lambda_i, \lambda_{i+1}]$ which 
means that we consider all orders of the items that can lead to optimal 
solutions when inserting the items greedily according to any fixed weighting 
of expected weights and variances.

In the following, we analyze  the time until an optimal solution with exactly $k$ elements has been produced for 
$f_{\lambda_i^*}(x) = \lambda_i^* \mu(x) + (1-\lambda_i^*) v(x) $
for any $k\in\{0,\dots,n\}$ and any $i\in\{0,\dots,\ell\}$. 
Note that these $\lambda_i^*$ values allow to obtain all optimal solutions for the set of functions $f_{\lambda}$, $\lambda \in [0,1]$.

For a given $i$, let the items be ordered such that $f_{\lambda_i^*}(e_1) \leq \dots \leq f_{\lambda_i^*}(e_k) \leq \dots \leq f_{\lambda_i^*}(e_n)$ holds.
An optimal solution for $k$ elements and $\lambda_i^*$ consists of $k$ elements with the smallest $f_{\lambda_i^*}(e_i)$-value. If there 
are more than one element with the value $f_{\lambda_i^*}(e_k)$ then reordering these elements does not change the objective vector or $f_{\lambda_i^*}$-value.

Note that for $k=0$ the search point $0^n$ is optimal for any $\lambda \in [0,1]$. 
Picking an optimal solution with $k$ elements for $f_{\lambda_i^*}$ and inserting an element with value $f_{\lambda_i^*}(e_{k+1})$ 
leads to an optimal solution for $f_{\lambda_i^*}$ with $k+1$ elements.
 We call such a step, picking the solution that is optimal for $f_{\lambda_i^*}$ with $k$ elements and inserting 
 an element with value $f_{\lambda_i^*}(e_{k+1})$, a success. Assuming such a solution is picked, the probability of inserting the 
 element is at least~$(1/n)(1-1/n)^{n-1}\ge 1/(en)$ 
 since it suffices 
 that \swgsemotD flips a specific bit and does not flip the rest.

     We now consider a sequence of events leading to the successes for 
     all values of~$k\in\{0,\dots,n-1\}$ and all $i\in\{0,\dots,\ell-1\}$. 
     We abbreviate $\pmax^*\coloneqq \max_{j=0}^{n-1} \pmax^{(j)}$. By the assumption from the theorem, 
    $0^n$, an optimal solution for~$k=0$, is in the population at 
    time~$0$. Assume that 
    optimal solutions with~$k$ elements all for $f_{\lambda_i^*}$, where $i\in\{0,\dots,\ell\}$, 
    are in the population~$P$ at time $\tau_k \coloneqq k\temax/n = 4e\pmax^* kn\ln n$. 
    
    Then, by definition of 
    the set $\hat{P}$ of \gsemotD, for any fixed~$i$, such a solution is available for selection 
    up to time \[\tau_{k+1}-1 = (k+1)\temax/n -1 = 4e\pmax^* (k+1)n\ln n-1\] since 
    $\lfloor ((k+1)\temax/n-1)/\temax)\cdot n \rfloor= k.$
    The size of the subset population that the algorithm selects from 
    during this period has been denoted by $\pmax^{(k)}$. 
    Therefore, the probability of a success at any fixed value~$k$ and~$i$ 
    is at least $1/(\pmax^{(k)}en)$ from time~$\tau_k$ until time $\tau_{k+1}-1$, \ie, for a period 
    of $4e \pmax^{*} n\ln n \ge 4e \pmax^{(k)} n\ln n
    $ steps, and the probability 
    of this not happening is at most 
    \[\left(1-\frac{1}{\pmax^{(k)} en}\right)^{4e\pmax^{(k)} n\ln n}\le \frac{1}{n^4}.\]

    The number~$\ell$ of different values of $\lambda_i^*$ is at most the number of pairs of elements and therefore at most $n^2$. By a union bound over this number of values and all~$k$, the probability to have not obtained all optimal solutions for all $f_{\lambda_i^*}$, where $i\in\{0,\dots,\ell\}$,  and all 
    values of~$k\in\{0,\dots,B\}$ by time~$\temax$  is $O(1/n)$. This shows the result for Equation~\eqref{chance-problem}.  
    The result for~\eqref{chance-problem2} follows from the proof of \cite[Theorem~4.3]{NeumannWittGECCO23}. 
\end{proof}

Finally, as mentioned above, we consider a uniform choice of the initial individual of \swgsemotD 
and show that the time to reach the all-zeros string is bounded by $O(n\log n)$ if the largest 
possible expected value $\mu_{\max}\coloneqq \sum_{i=1}^n \mu_i$ of an individual 
is polynomially bounded. Hence, this constitutes a lower-order term in terms of the optimization 
time bound proved in Theorem~\ref{theo:swgsemotdzeros} above. Even if $\mu_{\max}$ is exponential like $2^{n^c}$ for a constant~$c$, the bound of the lemma 
is still polynomial.

\begin{lemma}
    \label{lem:swgsmtotduniform}
Consider \swgsemotD initialized with a random bit string. Then the expected time until its population includes the all-zeros string for the first time is 
bounded from above by $O(n(\log \mu_{\max}+1))$. 
\end{lemma}

\begin{proof}
We apply multiplicative drift analysis \cite{DoerrJohannsenWinzenALGO12} with respect to the stochastic process $X_t\coloneqq \min\{\mu(x)\mid x\in P_t\}$, \ie, the minimum 
expected value of the individuals of the population at time~$t$. By definition, before the all-zeros string is included in the population, 
\swgsemotD chooses only individuals of minimum $\mu$-value for mutation. The current $\mu$-value of an individual 
is the sum of the expected values belonging to the 
bit positions that are set to~$1$. Standard-bit mutation flips each of these positions to~$0$ without flipping any other bit with probability at least~$(1/n)(1-1/n)^{n-1}\ge 1/(en)$. Such steps decrease the $\mu$-value of the 
solution, which is therefore not dominated by 
any other solution in the population and will be 
included afterwards. Hence, we obtain the drift $E(X_t - X_{t+1} \mid X_t) \ge X_t/(en)$. 
Using the parameter $\delta = 1/(en)$, $X_0\le \mu_{\max}$ and the fact that the smallest non-zero expected value of a bit is at least~$1$, we  
apply the multiplicative drift theorem \cite{DoerrJohannsenWinzenALGO12} and obtain an expected time of at most 
$\frac{\ln(\mu_{\max})+1}{\delta} = O(n(\log \mu_{\max}+1))$ to reach state~$0$ in the~$X_t$-process, \ie, an individual 
with all zeros. 
\end{proof}

\section{Experiments}
\label{sec:experiments}

We carry out experimental investigations for the new sliding window approach on the chance constrained dominating set problem and show where the new approach performs significantly better than the ones introduced in \cite{DBLP:conf/ijcai/0001W22,NeumannWittGECCO23}.

We recall the chance-constrained 
dominating set problem. 
The input is given as a graph $G=(V,E)$ with $n=|V|$ nodes and weights on the nodes.
The goal is to compute a set of nodes $D \subseteq V$ of minimal weight such that each node of the graph is dominated by $D$, i.e.\  either contained in $D$ or adjacent to a node in $D$. In our setting the weight $w_i$ of each node $v_i$ is chosen independently of the others according to a normal distribution $N(\mu_i, \sigma_i^2)$.
The constraint function $c(x)$ counts the number of nodes dominated by the given search point~$x$. As each node needs to be dominated in a feasible solution, $x$ is feasible iff $c(x)=n$ holds and therefore work with the bound $B=n$ in the algorithms.
We start with an initial solution $x \in \{0,1\}^n$ chosen uniformly at random.
We also investigate starting with $x = 0^n$ for Fast SW-GSEMO3D (denoted as Fast SW-GSEMO3D$_0$)  in the case of large graphs as this could be beneficial for such settings. We try to give some explanation by considering how the maximal population size differs when starting with a solution chosen uniformly at random or with~$0^n$.

\newgeometry{left=20mm, top=3mm} 

\begin{table*}[htbp]

 \caption{
    Results for stochastic minimum weight dominating set with different confidence levels of $\alpha$ where $\alpha=1-\beta$. Results after 10M fitness evaluations.
    $p_1$: Test GSEMO2D vs GSEMO3D, $p_2$: Test GSEMO2D vs Fast SW-GSEMO3D, $p_3$: Test GSEMO3D vs Fast SW-GSEMO3D, $p_4$: Fast GSEMO3D vs Fast SW-GSEMO3D$_0$.
    Penalty function value for run not obtaining a feasible solution is $10^{10}$ (applied to GSEMO3D for graphs ca-GrQc and Erdos992 )
    }
\tiny
    \centering
    \begin{tabular}{|p{10mm}|c||c|c|c|c|c||c|c|c|c|c|c|c|} \hline 
       \multirow{2}{*}{\shortstack{Graph/\\weight type}}  %\multirow{2}{*}{weight gype}  
       & \multirow{2}{*}{$\beta$} &  \multicolumn{2}{c|}{\bfseries GSEMO2D~\cite{NeumannWittGECCO23}} & \multicolumn{3}{c|}{\bfseries GSEMO3D~\cite{NeumannWittGECCO23}}  &  \multicolumn{4}{c|}{\bfseries Fast SW-GSEMO3D} &  \multicolumn{3}{c|}{\bfseries Fast SW-GSEMO3D$_0$} \\
   &  &    Mean & Std & Mean & Std & $p_1$-value &     Mean & Std & $p_2$-value & $p_3$-value &     Mean & Std & $p_4$-value \\ \hline

  \multirow{12}{*}{\shortstack{cfat200-1/\\uniform}} & % \multirow{12}{*}{uniform}
    0.2 & 3615 & 91 & 3599 & 79 & 0.544 & \textbf{3594} & 75 & 0.420 & 0.807 & 3598 & 74 & 0.767\\
&  0.1 & 3989 & 96 & 3972 & 80 & 0.544 & \textbf{3967} & 77 & 0.391 & 0.734 & 3971 & 79 & 0.745\\
&  0.01 & 4866 & 109 & 4845 & 86 & 0.535 & \textbf{4842} & 87 & 0.383 & 0.836 & 4846 & 90 & 0.784\\
&  1.0E-4 & 6015 & 126 & 5991 & 98 & 0.455 & \textbf{5989} & 101 & 0.412 & 0.894 & \textbf{5989} & 100 & 0.888\\
&  1.0E-6 & 6855 & 138 & 6832 & 108 & 0.605 & 6827 & 108 & 0.420 & 0.712 & \textbf{6825} & 107 & 0.848\\
&  1.0E-8 & 7546 & 147 & 7525 & 118 & 0.641 & 7517 & 115 & 0.455 & 0.723 & \textbf{7514} & 114 & 0.935\\
&  1.0E-10 & 8145 & 154 & 8125 & 125 & 0.751 & 8115 & 120 & 0.525 & 0.717 & \textbf{8112} & 120 & 0.853\\
&  1.0E-12 & 8680 & 159 & 8660 & 130 & 0.859 & 8651 & 126 & 0.615 & 0.790 & \textbf{8646} & 124 & 0.802\\
&  1.0E-14 & 9169 & 164 & 9148 & 133 & 0.842 & 9139 & 130 & 0.600 & 0.728 & \textbf{9133} & 128 & 0.830\\
%& & 1.0E-16 & 9611 & 168 & 9589 & 137 & 0.865 & 9580 & 134 & 0.589 & 0.728 & \textbf{9574} & 132 & 0.865\\
\hline
 \multirow{12}{*}{\shortstack{cfat200-2/\\uniform}} % &  \multirow{12}{*}{uniform}
 &  0.2 & 1791 & 49 & 1767 & 32 & 0.049 & 1766 & 33 & 0.031 & 0.712 & \textbf{1765} & 33 & 0.971\\
&  0.1 & 2040 & 54 & 2016 & 37 & 0.074 & 2014 & 36 & 0.044 & 0.819 & \textbf{2013} & 37 & 0.824\\
&  0.01 & 2621 & 72 & 2593 & 51 & 0.162 & 2588 & 49 & 0.066 & 0.610 & \textbf{2587} & 50 & 0.912\\
&  1.0E-4 & 3381 & 97 & 3336 & 65 & 0.070 & \textbf{3334} & 66 & 0.061 & 0.836 & \textbf{3334} & 67 & 0.947\\
&  1.0E-6 & 3937 & 113 & 3880 & 71 & 0.044 & \textbf{3879} & 75 & 0.036 & 0.853 & \textbf{3879} & 76 & 0.994\\
&  1.0E-8 & 4394 & 124 & 4329 & 77 & 0.032 & \textbf{4328} & 79 & 0.027 & 0.853 & \textbf{4328} & 79 & 1.000\\
&  1.0E-10 & 4793 & 132 & 4720 & 82 & 0.028 & 4719 & 82 & 0.021 & 0.877 & \textbf{4718} & 82 & 1.000\\
&  1.0E-12 & 5149 & 139 & 5071 & 85 & 0.024 & 5069 & 85 & 0.020 & 0.888 & \textbf{5068} & 85 & 0.988\\
&  1.0E-14 & 5475 & 145 & 5391 & 88 & 0.020 & \textbf{5389} & 87 & 0.016 & 0.912 & \textbf{5389} & 87 & 0.994\\
%& & 1.0E-16 & 5769 & 150 & 5681 & 91 & 0.021 & 5679 & 88 & 0.015 & 0.953 & \textbf{5678} & 89 & 0.971\\
\hline
 \multirow{12}{*}{\shortstack{ca-netscience/\\uniform}} % &  \multirow{12}{*}{uniform}
 &  0.2 & 33042 & 1289 & 33007 & 1023 & 0.712 & \textbf{32398} & 814 & 0.038 & 0.018 & 32399 & 861 & 0.976\\
&  0.1 & 34568 & 1302 & 34514 & 1028 & 0.745 & \textbf{33907} & 815 & 0.031 & 0.019 & 33908 & 865 & 0.988\\
&  0.01 & 38189 & 1334 & 38089 & 1040 & 0.848 & \textbf{37486} & 821 & 0.019 & 0.019 & 37489 & 874 & 0.941\\
&  1.0E-4 & 43012 & 1380 & 42846 & 1054 & 1.000 & \textbf{42248} & 841 & 0.011 & 0.020 & 42252 & 881 & 0.988\\
&  1.0E-6 & 46591 & 1415 & 46377 & 1065 & 0.824 & \textbf{45783} & 858 & 0.009 & 0.023 & 45786 & 888 & 0.882\\
&  1.0E-8 & 49557 & 1442 & 49303 & 1076 & 0.712 & \textbf{48712} & 870 & 0.008 & 0.021 & 48717 & 896 & 0.894\\
&  1.0E-10 & 52145 & 1465 & 51857 & 1087 & 0.615 & \textbf{51266} & 883 & 0.009 & 0.028 & 51275 & 906 & 0.935\\
&  1.0E-12 & 54467 & 1487 & 54150 & 1096 & 0.564 & \textbf{53557} & 894 & 0.007 & 0.028 & 53570 & 914 & 0.923\\
&  1.0E-14 & 56592 & 1507 & 56249 & 1105 & 0.487 & \textbf{55653} & 905 & 0.006 & 0.029 & 55670 & 923 & 0.912\\
%& & 1.0E-16 & 58517 & 1526 & 58151 & 1114 & 0.478 & \textbf{57552} & 915 & 0.005 & 0.026 & 57570 & 931 & 0.918\\
\hline
 \multirow{12}{*}{\shortstack{ca-GrQc\\/uniform}} % &  \multirow{12}{*}{uniform}
 &  0.2 & 5646101 & 79194 & 9666938258 & 1824254292 & 0.000 & \textbf{4920986} & 45094 & 0.000 & 0.000 & 4924856 & 40968 & 0.756\\
&  0.1 & 5712770 & 79494 & 9666940921 & 1824239705 & 0.000 & \textbf{4983403} & 45308 & 0.000 & 0.000 & 4987255 & 41159 & 0.756\\
&  0.01 & 5871104 & 80213 & 9666947246 & 1824205061 & 0.000 & \textbf{5131640} & 45823 & 0.000 & 0.000 & 5135447 & 41621 & 0.790\\
&  1.0E-4 & 6082155 & 81182 & 9666955677 & 1824158882 & 0.000 & \textbf{5329219} & 46511 & 0.000 & 0.000 & 5332980 & 42256 & 0.767\\
&  1.0E-6 & 6238913 & 81909 & 9666961940 & 1824124582 & 0.000 & \textbf{5475970} & 47028 & 0.000 & 0.000 & 5479688 & 42740 & 0.779\\
&  1.0E-8 & 6369023 & 82517 & 9666967137 & 1824096113 & 0.000 & \textbf{5597768} & 47451 & 0.000 & 0.000 & 5601451 & 43151 & 0.802\\
&  1.0E-10 & 6482579 & 83051 & 9666971674 & 1824071266 & 0.000 & \textbf{5704069} & 47822 & 0.000 & 0.000 & 5707719 & 43516 & 0.779\\
&  1.0E-12 & 6584589 & 83534 & 9666975749 & 1824048945 & 0.000 & \textbf{5799561} & 48160 & 0.000 & 0.000 & 5803181 & 43848 & 0.767\\
&  1.0E-14 & 6677976 & 83978 & 9666979480 & 1824028511 & 0.000 & \textbf{5886980} & 48471 & 0.000 & 0.000 & 5890573 & 44156 & 0.767\\
%& & 1.0E-16 & 6762658 & 84382 & 9666982863 & 1824009982 & 0.000 & \textbf{5966249} & 48759 & 0.000 & 0.000 & 5969819 & 44438 & 0.767\\
\hline
 \multirow{12}{*}{\shortstack{Erdos992/\\uniform}} & % \multirow{12}{*}{uniform}
   0.2 & 13716872 & 82588 & 10000000000 & 0 & 0.000 & 13482678 & 62860 & 0.000 & 0.000 & \textbf{13477560} & 55830 & 0.848\\
&  0.1 & 13842990 & 82789 & 10000000000 & 0 & 0.000 & 13607667 & 62812 & 0.000 & 0.000 & \textbf{13602550} & 55731 & 0.848\\
&  0.01 & 14142509 & 83278 & 10000000000 & 0 & 0.000 & 13904505 & 62706 & 0.000 & 0.000 & \textbf{13899386} & 55512 & 0.813\\
&  1.0E-4 & 14541754 & 83954 & 10000000000 & 0 & 0.000 & 14300178 & 62586 & 0.000 & 0.000 & \textbf{14295055} & 55242 & 0.790\\
&  1.0E-6 & 14838295 & 84474 & 10000000000 & 0 & 0.000 & 14594065 & 62512 & 0.000 & 0.000 & \textbf{14588938} & 55059 & 0.836\\
&  1.0E-8 & 15084429 & 84917 & 10000000000 & 0 & 0.000 & 14837996 & 62461 & 0.000 & 0.000 & \textbf{14832866} & 54917 & 0.836\\
&  1.0E-10 & 15299247 & 85313 & 10000000000 & 0 & 0.000 & 15050890 & 62423 & 0.000 & 0.000 & \textbf{15045759} & 54801 & 0.824\\
&  1.0E-12 & 15492221 & 85674 & 10000000000 & 0 & 0.000 & 15242138 & 62395 & 0.000 & 0.000 & \textbf{15237005} & 54703 & 0.836\\
&  1.0E-14 & 15668883 & 86011 & 10000000000 & 0 & 0.000 & 15417219 & 62375 & 0.000 & 0.000 & \textbf{15412085} & 54619 & 0.871\\
%& & 1.0E-16 & 15829079 & 86320 & 10000000000 & 0 & 0.000 & 15575981 & 62361 & 0.000 & 0.000 & \textbf{15570845} & 54546 & 0.836\\
\hline

 \multirow{12}{*}{\shortstack{cfat200-1/\\degree}} & % \multirow{12}{*}{degree} 
  0.2 & 4444 & 115 & \textbf{4387} & 6 & 0.001 & 4407 & 75 & 0.011 & 0.535 & 4398 & 55 & 0.779\\
&  0.1 & 4781 & 119 & \textbf{4721} & 9 & 0.003 & 4742 & 77 & 0.023 & 0.446 & 4733 & 56 & 0.790\\
&  0.01 & 5582 & 129 & \textbf{5512} & 16 & 0.004 & 5535 & 83 & 0.036 & 0.348 & 5526 & 61 & 0.819\\
&  1.0E-4 & 6650 & 143 & \textbf{6566} & 26 & 0.003 & 6592 & 91 & 0.035 & 0.287 & 6584 & 68 & 0.830\\
&  1.0E-6 & 7443 & 154 & \textbf{7349} & 34 & 0.003 & 7378 & 98 & 0.037 & 0.268 & 7369 & 74 & 0.830\\
&  1.0E-8 & 8101 & 163 & \textbf{7999} & 40 & 0.003 & 8029 & 103 & 0.041 & 0.268 & 8021 & 79 & 0.830\\
&  1.0E-10 & 8675 & 171 & \textbf{8567} & 45 & 0.003 & 8598 & 108 & 0.044 & 0.261 & 8590 & 84 & 0.865\\
&  1.0E-12 & 9191 & 178 & \textbf{9076} & 50 & 0.003 & 9109 & 113 & 0.043 & 0.261 & 9101 & 88 & 0.865\\
&  1.0E-14 & 9663 & 185 & \textbf{9542} & 55 & 0.003 & 9577 & 118 & 0.041 & 0.268 & 9569 & 92 & 0.853\\
%& & 1.0E-16 & 10091 & 191 & \textbf{9965} & 59 & 0.003 & 10001 & 122 & 0.043 & 0.255 & 9993 & 96 & 0.853\\
\hline
 \multirow{12}{*}{\shortstack{cfat200-2/\\degree}} % &  \multirow{12}{*}{degree} 
 & 0.2 & 3041 & 172 & \textbf{2963} & 4 & 0.027 & \textbf{2963} & 4 & 0.027 & 0.929 & \textbf{2963} & 4 & 0.830\\
&  0.1 & 3267 & 178 & \textbf{3185} & 6 & 0.027 & \textbf{3185} & 6 & 0.027 & 0.929 & \textbf{3185} & 6 & 0.830\\
&  0.01 & 3803 & 194 & 3713 & 11 & 0.027 & 3713 & 10 & 0.027 & 0.929 & \textbf{3712} & 10 & 0.830\\
&  1.0E-4 & 4518 & 216 & 4416 & 17 & 0.027 & \textbf{4415} & 16 & 0.027 & 0.929 & \textbf{4415} & 17 & 0.830\\
&  1.0E-6 & 5049 & 232 & 4938 & 22 & 0.027 & \textbf{4937} & 21 & 0.027 & 0.929 & \textbf{4937} & 21 & 0.830\\
&  1.0E-8 & 5490 & 245 & 5371 & 26 & 0.027 & 5371 & 24 & 0.027 & 0.929 & \textbf{5370} & 25 & 0.830\\
&  1.0E-10 & 5875 & 257 & 5749 & 30 & 0.027 & 5749 & 28 & 0.027 & 0.929 & \textbf{5748} & 28 & 0.830\\
&  1.0E-12 & 6220 & 267 & 6089 & 33 & 0.027 & 6088 & 30 & 0.027 & 0.929 & \textbf{6087} & 31 & 0.830\\
&  1.0E-14 & 6537 & 277 & 6400 & 36 & 0.027 & 6399 & 33 & 0.027 & 0.929 & \textbf{6398} & 34 & 0.830\\
%& & 1.0E-16 & 6823 & 286 & 6682 & 38 & 0.027 & 6681 & 35 & 0.027 & 0.929 & \textbf{6680} & 37 & 0.830\\
\hline
 \multirow{12}{*}{\shortstack{ca-netscience/\\degree}} %&  \multirow{12}{*}{degree} 
 & 0.2 & 28164 & 1002 & 26169 & 196 & 0.000 & \textbf{26097} & 197 & 0.000 & 0.017 & 26098 & 193 & 0.900\\
&  0.1 & 29689 & 1029 & 27657 & 200 & 0.000 & \textbf{27580} & 207 & 0.000 & 0.038 & 27583 & 201 & 0.853\\
&  0.01 & 33300 & 1098 & 31183 & 216 & 0.000 & \textbf{31092} & 238 & 0.000 & 0.114 & 31098 & 224 & 0.848\\
&  1.0E-4 & 38103 & 1192 & 35874 & 251 & 0.000 & \textbf{35758} & 284 & 0.000 & 0.092 & 35767 & 266 & 0.813\\
&  1.0E-6 & 41665 & 1265 & 39355 & 285 & 0.000 & \textbf{39220} & 324 & 0.000 & 0.076 & 39230 & 303 & 0.813\\
&  1.0E-8 & 44620 & 1327 & 42243 & 317 & 0.000 & \textbf{42091} & 359 & 0.000 & 0.067 & 42103 & 336 & 0.813\\
&  1.0E-10 & 47198 & 1381 & 44763 & 347 & 0.000 & \textbf{44596} & 390 & 0.000 & 0.067 & 44610 & 366 & 0.784\\
&  1.0E-12 & 49514 & 1429 & 47026 & 374 & 0.000 & \textbf{46845} & 418 & 0.000 & 0.074 & 46861 & 394 & 0.842\\
&  1.0E-14 & 51633 & 1474 & 49098 & 400 & 0.000 & \textbf{48905} & 444 & 0.000 & 0.081 & 48921 & 419 & 0.830\\
%& & 1.0E-16 & 53555 & 1515 & 50977 & 423 & 0.000 & \textbf{50772} & 469 & 0.000 & 0.076 & 50790 & 442 & 0.813\\
\hline
 \multirow{12}{*}{\shortstack{ca-GrQc/\\degree}} % &  \multirow{12}{*}{degree} 
 & 0.2 & 4032668 & 60538 & 9666845352 & 1824763160 & 0.000 & \textbf{3455870} & 17041 & 0.000 & 0.000 & 3457971 & 16176 & 0.460\\
&  0.1 & 4100297 & 61062 & 9666847956 & 1824748898 & 0.000 & \textbf{3517608} & 17204 & 0.000 & 0.000 & 3519680 & 16336 & 0.442\\
&  0.01 & 4260901 & 62312 & 9666854140 & 1824715027 & 0.000 & \textbf{3664186} & 17591 & 0.000 & 0.000 & 3666208 & 16722 & 0.442\\
&  1.0E-4 & 4474975 & 63984 & 9666862383 & 1824669878 & 0.000 & \textbf{3859529} & 18140 & 0.000 & 0.000 & 3861506 & 17249 & 0.408\\
&  1.0E-6 & 4633978 & 65230 & 9666868505 & 1824636344 & 0.000 & \textbf{4004604} & 18568 & 0.000 & 0.000 & 4006550 & 17648 & 0.460\\
&  1.0E-8 & 4765953 & 66266 & 9666873587 & 1824608510 & 0.000 & \textbf{4125007} & 18930 & 0.000 & 0.000 & 4126935 & 17988 & 0.469\\
&  1.0E-10 & 4881136 & 67173 & 9666878022 & 1824584217 & 0.000 & \textbf{4230080} & 19246 & 0.000 & 0.000 & 4232003 & 18291 & 0.460\\
&  1.0E-12 & 4984607 & 67988 & 9666882006 & 1824562394 & 0.000 & \textbf{4324470} & 19538 & 0.000 & 0.000 & 4326386 & 18569 & 0.451\\
&  1.0E-14 & 5079332 & 68736 & 9666885654 & 1824542416 & 0.000 & \textbf{4410880} & 19810 & 0.000 & 0.000 & 4412792 & 18828 & 0.478\\
%& & 1.0E-16 & 5165227 & 69415 & 9666888961 & 1824524301 & 0.000 & \textbf{4489236} & 20062 & 0.000 & 0.000 & 4491143 & 19067 & 0.478\\
\hline
 \multirow{12}{*}{\shortstack{Erdos992/\\degree}} % &  \multirow{12}{*}{degree}
 & 0.2 & 9307396 & 60880 & 10000000000 & 0 & 0.000 & 9104433 & 4932 & 0.000 & 0.000 & \textbf{9104421} & 4931 & 0.965\\
&  0.1 & 9433699 & 61228 & 10000000000 & 0 & 0.000 & 9229249 & 5100 & 0.000 & 0.000 & \textbf{9229244} & 4958 & 0.906\\
&  0.01 & 9733657 & 62061 & 10000000000 & 0 & 0.000 & \textbf{9525667} & 5566 & 0.000 & 0.000 & 9525686 & 5110 & 0.988\\
&  1.0E-4 & 10133488 & 63184 & 10000000000 & 0 & 0.000 & \textbf{9920775} & 6299 & 0.000 & 0.000 & 9920827 & 5490 & 0.953\\
&  1.0E-6 & 10430463 & 64027 & 10000000000 & 0 & 0.000 & \textbf{10214242} & 6902 & 0.000 & 0.000 & 10214318 & 5882 & 0.941\\
&  1.0E-8 & 10676958 & 64732 & 10000000000 & 0 & 0.000 & \textbf{10457822} & 7430 & 0.000 & 0.000 & 10457921 & 6265 & 0.988\\
&  1.0E-10 & 10892090 & 65351 & 10000000000 & 0 & 0.000 & \textbf{10670412} & 7907 & 0.000 & 0.000 & 10670530 & 6635 & 0.976\\
    &  1.0E-12 & 11085348 & 65911 & 10000000000 & 0 & 0.000 & \textbf{10861385} & 8347 & 0.000 & 0.000 & 10861521 & 6990 & 0.976\\
&  1.0E-14 & 11262269 & 66425 & 10000000000 & 0 & 0.000 & \textbf{11036215} & 8757 & 0.000 & 0.000 & 11036367 & 7332 & 1.000\\
%& & 1.0E-16 & 11422699 & 66893 & 10000000000 & 0 & 0.000 & \textbf{11194750} & 9134 & 0.000 & 0.000 & 11194916 & 7654 & 0.976\\
\hline

  \end{tabular}
   
    \label{tab:results}
\end{table*}

\newgeometry{left=15mm} 

\begin{table*}[htbp]
 \caption{
    Results for stochastic minimum weight dominating set with different confidence levels of $\alpha$ where $\alpha=1-\beta$. Results after 1M fitness evaluations.
    $p_1$: Test (1+1) EA vs GSEMO2D, $p_2$: Test (1+1) EA vs Fast SW-GSEMO3D, $p_3$: Test GSEMO2D vs Fast SW-GSEMO3D, $p_4$: Test (1+1) EA vs Fast SW-GSEMO3D$_0$, $p_5$: Test GSEMO2D vs Fast SW-GSEMO3D$_0$, $p_6$: Test Fast GSEMO3D vs Fast SW-GSEMO3D$_0$.
    }
\tiny
%\scriptsize
    \centering
    \begin{tabular}{|p{9mm}|c||c|c|c|c|c||c|c|c|c|c|c|c|c|c|} \hline 
       \multirow{2}{*}{\shortstack{Graph/\\weight type}} %& \multirow{2}{*}{weight gype} 
       & \multirow{2}{*}{$\beta$} &  \multicolumn{2}{c|}{\bfseries (1+1) EA~\cite{DBLP:conf/ijcai/0001W22}} & \multicolumn{3}{c|}{\bfseries GSEMO2D~\cite{DBLP:conf/ijcai/0001W22,NeumannWittGECCO23}}  &  \multicolumn{4}{c|}{\bfseries Fast SW-GSEMO3D} &  \multicolumn{5}{c|}{\bfseries Fast SW-GSEMO3D$_0$} \\
 &   &    Mean & Std & Mean & Std & $p_1$-val &     Mean & Std & $p_2$-val & $p_3$-val &     Mean & Std & $p_4$-val & $p_5$-val & $p_6$-val \\ \hline

  \multirow{12}{*}{\shortstack{ca-CSphd/\\uniform}} %&  \multirow{12}{*}{} 
  & 0.2 & 1176951 & 29560 & 1149185 & 21187 & 0.000 & 1053428 & 5919 & 0.000 & 0.000 & \textbf{1052480} & 4910 & 0.000 & 0.000 & 0.367\\
&  0.1 & 1200964 & 25599 & 1173498 & 21419 & 0.000 & 1076406 & 5973 & 0.000 & 0.000 & \textbf{1075454} & 4965 & 0.000 & 0.000 & 0.367\\
&  0.01 & 1235668 & 29329 & 1231241 & 21969 & 0.836 & 1130976 & 6108 & 0.000 & 0.000 & \textbf{1130017} & 5105 & 0.000 & 0.000 & 0.383\\
&  1.0E-4 & 1314570 & 28190 & 1308208 & 22705 & 0.451 & 1203715 & 6301 & 0.000 & 0.000 & \textbf{1202747} & 5308 & 0.000 & 0.000 & 0.375\\
&  1.0E-6 & 1378890 & 25618 & 1365376 & 23254 & 0.062 & 1257743 & 6455 & 0.000 & 0.000 & \textbf{1256767} & 5471 & 0.000 & 0.000 & 0.391\\
&  1.0E-8 & 1410240 & 22358 & 1412826 & 23711 & 0.712 & 1302586 & 6587 & 0.000 & 0.000 & \textbf{1301605} & 5612 & 0.000 & 0.000 & 0.383\\
&  1.0E-10 & 1455663 & 21030 & 1454239 & 24110 & 0.894 & 1341724 & 6707 & 0.000 & 0.000 & \textbf{1340738} & 5740 & 0.000 & 0.000 & 0.375\\
&  1.0E-12 & 1495936 & 29008 & 1491441 & 24470 & 0.574 & 1376883 & 6818 & 0.000 & 0.000 & \textbf{1375892} & 5859 & 0.000 & 0.000 & 0.433\\
&  1.0E-14 & 1526403 & 25752 & 1525499 & 24799 & 1.000 & 1409069 & 6921 & 0.000 & 0.000 & \textbf{1408074} & 5970 & 0.000 & 0.000 & 0.469\\
%& & 1.0E-16 & 1560289 & 30787 & 1556382 & 25098 & 0.712 & 1438256 & 7017 & 0.000 & 0.000 & \textbf{1437257} & 6073 & 0.000 & 0.000 & 0.442\\
\hline
 \multirow{12}{*}{\shortstack{ca-HepPh/\\uniform}} % &  \multirow{12}{*}{uniform} 
 & 0.2 & 24866045 & 323815 & 24664260 & 251849 & 0.010 & 21903190 & 229592 & 0.000 & 0.000 & \textbf{21655867} & 211163 & 0.000 & 0.000 & 0.000\\
&  0.1 & 25126756 & 223438 & 24941951 & 253168 & 0.006 & 22162387 & 230935 & 0.000 & 0.000 & \textbf{21913353} & 212217 & 0.000 & 0.000 & 0.000\\
&  0.01 & 25709929 & 219138 & 25601440 & 256304 & 0.101 & 22777957 & 234129 & 0.000 & 0.000 & \textbf{22524858} & 214726 & 0.000 & 0.000 & 0.000\\
&  1.0E-4 & 26602650 & 271535 & 26480507 & 260496 & 0.132 & 23598486 & 238398 & 0.000 & 0.000 & \textbf{23339968} & 218088 & 0.000 & 0.000 & 0.000\\
&  1.0E-6 & 27104133 & 304517 & 27133437 & 263618 & 0.595 & 24207935 & 241578 & 0.000 & 0.000 & \textbf{23945393} & 220598 & 0.000 & 0.000 & 0.000\\
&  1.0E-8 & 27675018 & 335011 & 27675380 & 266213 & 0.953 & 24713788 & 244221 & 0.000 & 0.000 & \textbf{24447901} & 222696 & 0.000 & 0.000 & 0.000\\
&  1.0E-10 & 28123068 & 314336 & 28148371 & 268482 & 0.941 & 25155281 & 246532 & 0.000 & 0.000 & \textbf{24886470} & 224540 & 0.000 & 0.000 & 0.000\\
&  1.0E-12 & 28616742 & 357514 & 28573268 & 270522 & 0.636 & 25551883 & 248611 & 0.000 & 0.000 & \textbf{25280441} & 226199 & 0.000 & 0.000 & 0.000\\
&  1.0E-14 & 28831138 & 286317 & 28962248 & 272392 & 0.143 & 25914960 & 250516 & 0.000 & 0.000 & \textbf{25641110} & 227723 & 0.000 & 0.000 & 0.000\\
%& & 1.0E-16 & 29163100 & 306357 & 29314972 & 274089 & 0.043 & 26244195 & 252245 & 0.000 & 0.000 & \textbf{25968162} & 229106 & 0.000 & 0.000 & 0.000\\
\hline
 \multirow{12}{*}{\shortstack{ca-AstroPh/\\uniform}} % &  \multirow{12}{*}{uniform} 
 & 0.2 & 51557918 & 568600 & 51043030 & 528254 & 0.001 & 64103184 & 4470490 & 0.000 & 0.000 & \textbf{45226809} & 500442 & 0.000 & 0.000 & 0.000\\
&  0.1 & 51942457 & 555700 & 51548678 & 531285 & 0.021 & 64668884 & 4491484 & 0.000 & 0.000 & \textbf{45698407} & 502905 & 0.000 & 0.000 & 0.000\\
&  0.01 & 53161346 & 658583 & 52749539 & 538490 & 0.017 & 66012371 & 4541343 & 0.000 & 0.000 & \textbf{46818411} & 508759 & 0.000 & 0.000 & 0.000\\
&  1.0E-4 & 54581672 & 577272 & 54350226 & 548110 & 0.160 & 67803180 & 4607807 & 0.000 & 0.000 & \textbf{48311327} & 516571 & 0.000 & 0.000 & 0.000\\
&  1.0E-6 & 55574306 & 568036 & 55539139 & 555269 & 0.965 & 69133305 & 4657175 & 0.000 & 0.000 & \textbf{49420191} & 522386 & 0.000 & 0.000 & 0.000\\
&  1.0E-8 & 56482376 & 659036 & 56525957 & 561218 & 0.525 & 70237331 & 4698155 & 0.000 & 0.000 & \textbf{50340566} & 527216 & 0.000 & 0.000 & 0.000\\
&  1.0E-10 & 56997947 & 442067 & 57387223 & 566415 & 0.003 & 71200892 & 4733922 & 0.000 & 0.000 & \textbf{51143842} & 531435 & 0.000 & 0.000 & 0.000\\
&  1.0E-12 & 58002729 & 535712 & 58160914 & 571088 & 0.255 & 72066476 & 4766054 & 0.000 & 0.000 & \textbf{51865440} & 535228 & 0.000 & 0.000 & 0.000\\
&  1.0E-14 & 58598177 & 480173 & 58869203 & 575369 & 0.033 & 72858892 & 4795471 & 0.000 & 0.000 & \textbf{52526040} & 538702 & 0.000 & 0.000 & 0.000\\
%& & 1.0E-16 & 59336487 & 480899 & 59511475 & 579253 & 0.183 & 73577448 & 4822146 & 0.000 & 0.000 & \textbf{53125068} & 541854 & 0.000 & 0.000 & 0.000\\
\hline
 \multirow{12}{*}{\shortstack{ca-CondMat/\\uniform}} % &  \multirow{12}{*}{uniform} 
 & 0.2 & 87564936 & 940507 & 86293144 & 783450 & 0.000 & 431800766 & 1807172824 & 0.000 & 0.000 & \textbf{75931086} & 610598 & 0.000 & 0.000 & 0.000\\
&  0.1 & 87993459 & 758163 & 87014750 & 786716 & 0.000 & 432555511 & 1807030509 & 0.000 & 0.000 & \textbf{76602241} & 613185 & 0.000 & 0.000 & 0.000\\
&  0.01 & 89127748 & 754815 & 88728501 & 794478 & 0.069 & 434347964 & 1806692530 & 0.000 & 0.000 & \textbf{78196177} & 619334 & 0.000 & 0.000 & 0.000\\
&  1.0E-4 & 91086972 & 739979 & 91012856 & 804836 & 0.859 & 436737226 & 1806242026 & 0.000 & 0.000 & \textbf{80320825} & 627546 & 0.000 & 0.000 & 0.000\\
&  1.0E-6 & 92467544 & 650611 & 92709566 & 812539 & 0.204 & 438511855 & 185907420 & 0.000 & 0.000 & \textbf{81898913} & 633655 & 0.000 & 0.000 & 0.000\\
&  1.0E-8 & 93588939 & 815736 & 94117866 & 818937 & 0.015 & 439984829 & 1805629695 & 0.000 & 0.000 & \textbf{83208753} & 638731 & 0.000 & 0.000 & 0.000\\
&  1.0E-10 & 94695345 & 520061 & 95346987 & 824526 & 0.002 & 441270396 & 1805387308 & 0.000 & 0.000 & \textbf{84351942} & 643167 & 0.000 & 0.000 & 0.000\\
&  1.0E-12 & 96086744 & 975803 & 96451130 & 829550 & 0.183 & 442425245 & 1805169569 & 0.000 & 0.000 & \textbf{85378890} & 647155 & 0.000 & 0.000 & 0.000\\
&  1.0E-14 & 96686021 & 889063 & 97461938 & 834151 & 0.001 & 443482473 & 1804970239 & 0.000 & 0.000 & \textbf{86319029} & 650809 & 0.000 & 0.000 & 0.000\\
%& & 1.0E-16 & 97660223 & 830935 & 98378530 & 838326 & 0.002 & 444441159 & 1804789489 & 0.000 & 0.000 & \textbf{87171539} & 654125 & 0.000 & 0.000 & 0.000\\
\hline

 \multirow{12}{*}{\shortstack{ca-CSphd/\\degree}} %&  \multirow{12}{*}{degree} 
 & 0.2 & 1176359 & 23453 & 1166190 & 32090 & 0.071 & 1053397 & 6005 & 0.000 & 0.000 & \textbf{1052796} & 5364 & 0.000 & 0.000 & 0.668\\
&  0.1 & 1197763 & 27695 & 1190714 & 32418 & 0.322 & 1076425 & 6076 & 0.000 & 0.000 & \textbf{1075804} & 5419 & 0.000 & 0.000 & 0.657\\
&  0.01 & 1243411 & 22570 & 1248957 & 33200 & 0.416 & 1131114 & 6256 & 0.000 & 0.000 & \textbf{1130446} & 5555 & 0.000 & 0.000 & 0.647\\
&  1.0E-4 & 1318313 & 23041 & 1326592 & 34244 & 0.294 & 1204011 & 6514 & 0.000 & 0.000 & \textbf{1203281} & 5748 & 0.000 & 0.000 & 0.615\\
&  1.0E-6 & 1370672 & 30267 & 1384255 & 35022 & 0.110 & 1258155 & 6718 & 0.000 & 0.000 & \textbf{1257380} & 5898 & 0.000 & 0.000 & 0.615\\
&  1.0E-8 & 1411274 & 28448 & 1432117 & 35668 & 0.004 & 1303096 & 6895 & 0.000 & 0.000 & \textbf{1302282} & 6027 & 0.000 & 0.000 & 0.605\\
&  1.0E-10 & 1465714 & 31864 & 1473889 & 36233 & 0.322 & 1342319 & 7054 & 0.000 & 0.000 & \textbf{1341472} & 6143 & 0.000 & 0.000 & 0.584\\
&  1.0E-12 & 1494845 & 26265 & 1511414 & 36742 & 0.076 & 1377554 & 7202 & 0.000 & 0.000 & \textbf{1376676} & 6249 & 0.000 & 0.000 & 0.595\\
&  1.0E-14 & 1539841 & 28989 & 1545767 & 37207 & 0.756 & 1409811 & 7339 & 0.000 & 0.000 & \textbf{1408905} & 6349 & 0.000 & 0.000 & 0.584\\
%& & 1.0E-16 & 1563626 & 33155 & 1576918 & 37630 & 0.046 & 1439061 & 7466 & 0.000 & 0.000 & \textbf{1438130} & 6440 & 0.000 & 0.000 & 0.554\\
\hline
 \multirow{12}{*}{\shortstack{ca-HepPh/\\degree}} % &  \multirow{12}{*}{degree} 
 & 0.2 & 24940255 & 229915 & 24770247 & 328453 & 0.019 & 21925184 & 256481 & 0.000 & 0.000 & \textbf{21672753} & 170643 & 0.000 & 0.000 & 0.000\\
&  0.1 & 25129650 & 329488 & 25048454 & 330378 & 0.322 & 22184365 & 258158 & 0.000 & 0.000 & \textbf{21930197} & 171670 & 0.000 & 0.000 & 0.000\\
&  0.01 & 25755684 & 291735 & 25709164 & 334958 & 0.584 & 22799898 & 262146 & 0.000 & 0.000 & \textbf{22541605} & 174114 & 0.000 & 0.000 & 0.000\\
&  1.0E-4 & 26478853 & 313950 & 26589855 & 341060 & 0.156 & 23620377 & 267470 & 0.000 & 0.000 & \textbf{23356586} & 177388 & 0.000 & 0.000 & 0.000\\
&  1.0E-6 & 27073736 & 305766 & 27243988 & 345599 & 0.055 & 24229788 & 271431 & 0.000 & 0.000 & \textbf{23961915} & 179829 & 0.000 & 0.000 & 0.000\\
&  1.0E-8 & 27647166 & 283416 & 27786927 & 349368 & 0.086 & 24735610 & 274722 & 0.000 & 0.000 & \textbf{24464348} & 181862 & 0.000 & 0.000 & 0.000\\
&  1.0E-10 & 28101126 & 327539 & 28260785 & 352656 & 0.079 & 25177076 & 277597 & 0.000 & 0.000 & \textbf{24902857} & 183641 & 0.000 & 0.000 & 0.000\\
&  1.0E-12 & 28523939 & 323332 & 28686461 & 355612 & 0.071 & 25573654 & 280182 & 0.000 & 0.000 & \textbf{25296778} & 185243 & 0.000 & 0.000 & 0.000\\
&  1.0E-14 & 28937484 & 306489 & 29076155 & 358320 & 0.147 & 25936707 & 282550 & 0.000 & 0.000 & \textbf{25657400} & 186712 & 0.000 & 0.000 & 0.000\\
%& & 1.0E-16 & 29241416 & 250998 & 29429525 & 360777 & 0.023 & 26265920 & 284699 & 0.000 & 0.000 & \textbf{25984409} & 188047 & 0.000 & 0.000 & 0.000\\
\hline
 \multirow{12}{*}{\shortstack{ca-AstroPh/\\degree}} % &  \multirow{12}{*}{degree} 
 & 0.2 & 51524407 & 570578 & 50681144 & 611971 & 0.000 & 64564376 & 6120887 & 0.000 & 0.000 & \textbf{45109042} & 578792 & 0.000 & 0.000 & 0.000\\
&  0.1 & 52090421 & 613178 & 51184838 & 615193 & 0.000 & 65131764 & 6149639 & 0.000 & 0.000 & \textbf{45579883} & 581545 & 0.000 & 0.000 & 0.000\\
&  0.01 & 53271848 & 477150 & 52381067 & 622852 & 0.000 & 66479261 & 6217924 & 0.000 & 0.000 & \textbf{46698084} & 588091 & 0.000 & 0.000 & 0.000\\
&  1.0E-4 & 54408644 & 498038 & 53975585 & 633070 & 0.012 & 68275417 & 6308950 & 0.000 & 0.000 & \textbf{48188591} & 596839 & 0.000 & 0.000 & 0.000\\
&  1.0E-6 & 55533826 & 541501 & 55159915 & 640666 & 0.015 & 69609515 & 6376561 & 0.000 & 0.000 & \textbf{49295668} & 603348 & 0.000 & 0.000 & 0.000\\
&  1.0E-8 & 56254153 & 556558 & 56142930 & 646976 & 0.469 & 70716840 & 6432682 & 0.000 & 0.000 & \textbf{50214561} & 608758 & 0.000 & 0.000 & 0.000\\
&  1.0E-10 & 57221946 & 419431 & 57000876 & 652487 & 0.165 & 71683280 & 6481663 & 0.000 & 0.000 & \textbf{51016543} & 613485 & 0.000 & 0.000 & 0.000\\
&  1.0E-12 & 58115722 & 620857 & 57771583 & 657440 & 0.048 & 72551451 & 6525664 & 0.000 & 0.000 & \textbf{51736978} & 617736 & 0.000 & 0.000 & 0.000\\
&  1.0E-14 & 58619146 & 587556 & 58477140 & 661975 & 0.433 & 73346234 & 6565947 & 0.000 & 0.000 & \textbf{52396514} & 621630 & 0.000 & 0.000 & 0.000\\
%& & 1.0E-16 & 59197783 & 523162 & 59116933 & 666089 & 0.690 & 74066938 & 6602475 & 0.000 & 0.000 & \textbf{52994576} & 625164 & 0.000 & 0.000 & 0.000\\
\hline
 \multirow{12}{*}{\shortstack{ca-CondMat/\\degree}} % &  \multirow{12}{*}{degree} 
 & 0.2 & 87579791 & 869378 & 86547921 & 742313 & 0.000 & 103694132 & 9559675 & 0.000 & 0.000 & \textbf{75939104} & 868161 & 0.000 & 0.000 & 0.000\\
&  0.1 & 87713127 & 806685 & 87270373 & 745665 & 0.028 & 104482533 & 9598957 & 0.000 & 0.000 & \textbf{76609900} & 872129 & 0.000 & 0.000 & 0.000\\
&  0.01 & 89207002 & 639620 & 88986134 & 753635 & 0.132 & 106354916 & 9692249 & 0.000 & 0.000 & \textbf{78202983} & 881557 & 0.000 & 0.000 & 0.000\\
&  1.0E-4 & 90856334 & 857580 & 91273168 & 764272 & 0.110 & 108850723 & 9816605 & 0.000 & 0.000 & \textbf{80326490} & 894132 & 0.000 & 0.000 & 0.000\\
&  1.0E-6 & 92392083 & 865712 & 92971867 & 772182 & 0.017 & 110704488 & 9908973 & 0.000 & 0.000 & \textbf{81903730} & 903477 & 0.000 & 0.000 & 0.000\\
&  1.0E-8 & 93674413 & 808337 & 94381818 & 778752 & 0.003 & 112243147 & 9985641 & 0.000 & 0.000 & \textbf{83212867} & 911237 & 0.000 & 0.000 & 0.000\\
&  1.0E-10 & 94664923 & 600565 & 95612380 & 784491 & 0.000 & 113586040 & 10052555 & 0.000 & 0.000 & \textbf{84355442} & 918012 & 0.000 & 0.000 & 0.000\\
&  1.0E-12 & 96145670 & 762399 & 96717817 & 789649 & 0.017 & 114792388 & 10112666 & 0.000 & 0.000 & \textbf{85381839} & 924101 & 0.000 & 0.000 & 0.000\\
&  1.0E-14 & 96814075 & 636459 & 97729809 & 794374 & 0.000 & 115896760 & 10167696 & 0.000 & 0.000 & \textbf{86321473} & 929676 & 0.000 & 0.000 & 0.000\\
%& & 1.0E-16 & 97849258 & 947857 & 98647476 & 798660 & 0.001 & 116898197 & 10217598 & 0.000 & 0.000 & \textbf{87173526} & 934734 & 0.000 & 0.000 & 0.000\\
\hline

  \end{tabular}
   
    \label{tab:results2}
\end{table*}

\restoregeometry

\begin{table*}[t]
\caption{Average maximum population size and standard deviation during the 30 runs of 1M iterations for Fast SW-GSEMO3D and Fast SW-GSEMO3D$_0$ in the uniform random, degree-based setting for large graphs.}
\scriptsize
    \centering
    \begin{tabular}{|c|c|c|c|c|c|c|c|c|} \hline 
       \multirow{3}{*}{Graph} & \multicolumn{4}{|c|}{\bfseries Fast SW-GSEMO3D} & \multicolumn{4}{|c|}{\bfseries Fast SW-GSEMO3D$_0$}  \\ \hline
       
   &    \multicolumn{2}{|c|}{\bfseries uniform} & \multicolumn{2}{|c|}{\bfseries degree} & \multicolumn{2}{|c|}{\bfseries uniform} & \multicolumn{2}{|c|}{\bfseries degree}   \\ 
     & Mean & Std  & Mean &  Std & Mean & Std & Mean &  Std \\ \hline
    \multirow{1}{*}{ca-CSphd} &  665 & 40.555 & 670 & 37.727 & 225 & 16.829 & 230 & 15.616\\
\hline
 \multirow{1}{*}{ca-HepPh} &  2770 & 124.786 & 2713 & 166.804 & 125 & 15.561 & 128 & 20.372\\
\hline
 \multirow{1}{*}{ca-AstroPh} &  3608 & 167.880 & 3602 & 132.344 & 140 & 26.422 & 144 & 25.647\\
\hline
 \multirow{1}{*}{ca-CondMat} &  5196 & 130.968 & 5245 & 109.568 & 107 & 20.817 & 104 & 19.662\\
\hline
    \end{tabular}
    \vspace{-\medskipamount}
    \label{tab:popsize}
\end{table*}

As done in~\cite{DBLP:conf/ijcai/0001W22,NeumannWittGECCO23}, we consider the graphs cfat200-1, cfat200-2, ca-netscience, ca-GrQc, and Erdos992 consisting of $200$, $200$, $379$, $4158$, and $6100$ nodes respectively, together with 
the following categories for choosing the weights.
%as done in \cite{DBLP:conf/ijcai/0001W22,NeumannWittGECCO23}.
In the \emph{uniform} setting each weight $\mu(u)$ is an integer chosen independently and uniformly at random in $\{n, \ldots, 2 n\}$. The variance $v(u)$ is an integer chosen independently and uniformly at random in $\{n^2, \ldots, 2n^2\}$.
In the \emph{degree-based} setting, we have $\mu(u)= (n + \deg(u))^5/n^4$ where $\deg(u)$ is the degree of node $u$ in the given graph. The variance $v(u)$ is an integer chosen independently and uniformly at random in $\{n^2, \ldots, 2n^2\}$. 
For these graphs, we use 10M (million) fitness evaluations for each run.
We also use the graphs ca-CSphd, ca-HepPh, ca-AstroPh, ca-CondMat, which consist of 1882, 11204, 17903, 21363 nodes. They have already been investigated in \cite{NeumannWittECAI23} in the context of the maximum coverage problem. We examine the same uniform random and degree-based setting as described before. We consider 1M fitness evaluations for these graphs in order to investigate the performance on large graphs with a smaller fitness evaluation budget.

For our new sliding window algorithms we use $\tefrac=0.9$, $std=10$, $a=0.5$, $\epsilon=0$ based on some preliminary experimental investigations. 
Furthermore, we consider $10M$ fitness evaluations for all algorithms and results presented in Table~\ref{tab:results} and 1M fitness evaluations for the instances in Table~\ref{tab:results2}.
For each setting, each considered algorithm is run on the same set of $30$ randomly generated instances.
 We use the Mann-Whitney test to compute the $p$-value for algorithm pairs to check whether results are statistically significant,  which we assume to be the case if the p-value is at most 0.05.

%\subsection{Experimental results}

We first consider results for the instances already investigated in~\cite{NeumannWittGECCO23}. 
We consider the random and degree based instances and results for the examined algorithms are shown in Table~\ref{tab:results}. Results for the GSEMO2D approach developed in \cite{DBLP:conf/ijcai/0001W22} and the GSEMO3D developed in \cite{NeumannWittGECCO23} have already been obtained in ~\cite{NeumannWittGECCO23}. Each run that does not obtain a dominating set gets allocated a fitness value of $10^{10}$. We note that this only applies to GSEMO3D for ca-GrQc and Erdos992 and GSEMO3D. It has already been stated in~\cite{NeumannWittGECCO23} that GSEMO3D has difficulties in obtaining feasible solutions for these graphs. In fact, it never returns a feasible solution for  Erdos992 in both chance constrained settings and only in 1 out of 30 runs for ca-GrQc in both chance constrained settings.
Comparing the results of GSEMO2D and GSEMO3D to our new approaches Fast SW-GSEMO3D and Fast SW-GSEMO3D$_0$, we can see that all approaches behave quite similar for cfat200-1 and cfat200-2. For ca-netscience, there is a slight advantage for our fast sliding window approaches that is statistically significant when compared to GSEMO2D and GSEMO3D, but no real difference on whether the sliding window approach starts with an initial solution chosen uniformly at random or with the search point $0^n$.
Both Fast SW-GSEMO3D and Fast SW-GSEMO3D$_0$ show their real advantage for the larger graphs ca-GrQc and Erdos992 where the 3-objective approach GSEMO3D was unable to produce feasible solutions. On these instance GSEMO2D is clearly outperformed by the sliding window 3-objective approaches.

Results for the instances based on the graphs ca-CSphd, ca-HepPh, ca-AstroPh, ca-CondMat, which consist of 1882, 4158, 6100, 11204, 17903, 21363 nodes are shown in Table~\ref{tab:results2}. Note that the graphs (except ca-CSphd) have more than 10000 nodes and are therefore significantly larger than the ones tested previously. As we are dealing with larger graphs and a smaller fitness evaluation budget of 1M, we also consider the (1+1)~EA approach presented in \cite{DBLP:conf/ijcai/0001W22}. Here each run of the (1+1)~EA tackles each value of $\alpha$ (see Equation~\eqref{eq:weight}) separately with a budget of 1M fitness evaluations, which implies the single-objective approach uses a fitness evaluation budget that is ten times the one of the multi-objective approaches. We observe that  Fast SW-GSEMO3D$_0$ overall produces the best results. For the smallest graph ca-CSphd, there is no significant difference on whether to start with an initial solution chosen uniformly at random or with the search point $0^n$. However, for the larger graphs ca-HepPh, ca-AstroPh, ca-CondMat consisting of more than 10000 nodes, starting with the initial search point $0^n$ in the sliding window approach is crucial for the success of the algorithm. In particular, we can observe that Fast SW-GSEMO3D starting uniformly at random is performing significantly worse than the (1+1) EA and GSEMO2D for the graphs ca-AstroPh, ca-CondMat consisting of 17903 and 21363 nodes, respectively. All observations hold for the uniform as well as the degree-based chance constrained settings.

As mentioned starting with  $0^n$ in our sliding window approach provides a clear benefit when dealing with large graphs. We have already seen in our analysis that the sliding window approach starts at the constraint value of $0$ which gives a partial explanation of its benefit. In order to gain additional insights, we provide in Table~\ref{tab:popsize} the maximum population sizes that the approaches Fast SW-GSEMO3D and  Fast SW-GSEMO3D$_0$ encounter for the graphs ca-CSphd, ca-HepPh, ca-AstroPh, ca-CondMat. We can observe that the maximum population sizes when starting with the search point $0^n$ are significantly smaller than when starting with an initial solution chosen uniformly at random. For the graph ca-CondMat, the average maximum population size among the executed 30 runs for Fast SW-GSEMO3D is almost by a factor of 50 larger than for Fast SW-GSEMO3D$_0$ (5196 vs.\ 107). Given that large populations can significantly slow down the progress of the sliding window approach, we regard the difference in maximum population sizes as a clear explanation why Fast SW-GSEMO3D$_0$ clearly outperforms Fast SW-GSEMO3D on the graphs ca-HepPh, ca-AstroPh, and ca-CondMat.
%% the bibliography file.

\section*{Conclusions}
We have shown how to significantly speed and scale up the $3$\nobreakdash-ob\-jec\-tive approach for chance constrained problems introduced in \cite{NeumannWittGECCO23}. We have presented a sliding window approach and shown that it provides with high probability the same theoretical approximation quality as the one given in \cite{NeumannWittGECCO23} but within a significantly smaller fitness evaluation budget. Our experimental investigations show that the new approach is able to deal with chance constrained instances of the dominating set problem with up to 20,000 nodes (within 1M iterations) whereas the previous approach given in \cite{NeumannWittGECCO23} was not able to produce good quality (or even feasible) solutions for already medium size instances of around 4,000 nodes (within 10M iterations).

\noindent\textbf{Acknowledgments.}
This work has been supported by the Australian Research Council (ARC) through grant
FT200100536 and by the Independent Research Fund Denmark through grant DFF-FNU 8021-00260B.

\end{document}